\documentclass[10pt,conference]{IEEEtran} 
\usepackage[utf8]{inputenc}
\newcommand{\id}{\mathbf{I}} % big i for identity
\newcommand{\0}{\mathbf{0}} % the origin
\newcommand{\x}{{\bf x}}

\newcommand{\jac}[1]{\mathbf{J}_{#1}}

\newcommand{\reals}{\mathbb{R}}

\newcommand{\bq}{\begin{equation}}
\newcommand{\eq}{\end{equation}}
\newcommand{\ba}{\begin{eqnarray}}
\newcommand{\ea}{\end{eqnarray}}

\newcommand{\originalvf}{f^{\text{real}}}
\newcommand{\ltwonorm}[1]{\|#1\|_2}
\newcommand{\sym}{\operatorname{{\bf sym}}}
\newcommand{\dt}{\;{\rm d}t}
\newcommand{\iu}{{\bf u}}

\newcommand{\samplex}{{\x^{\text{sample}}}}

\def\R{{\reals}}

\usepackage[utf8]{inputenc}
\usepackage{sidecap}
\usepackage{paralist}
\usepackage{times}
\usepackage{algorithm}
\usepackage{algorithmic}
\usepackage{float}
\usepackage{graphicx}
\usepackage{amssymb}
\usepackage{media9}
\usepackage{caption,tabularx,booktabs}
\usepackage{array}
\usepackage[font=small,labelfont=bf,tableposition=top]{caption}
\usepackage[font=footnotesize]{subcaption}
\usepackage{url}
\usepackage{mathtools}
\usepackage[dvipsnames]{xcolor}
\usepackage{multicol}
\usepackage{graphicx,booktabs,array}
\usepackage[bookmarks=true]{hyperref}

\newcolumntype{H}{>{\setbox0=\hbox\bgroup}c<{\egroup}@{}}
\graphicspath{ {./code/} }

\newtheorem{definition}{Definition}
\newtheorem{proposition}{Proposition}
\newtheorem{lemma}{Lemma}

\newtheorem{remark}{Remark}
\newtheorem{theorem}{Theorem}

\let\labelindent\relax

\usepackage{enumitem}
\usepackage{multirow}
\newcommand{\bachir}[1]{{#1}}

\IEEEoverridecommandlockouts

\title{\LARGE \bf Teleoperator Imitation with Continuous-time Safety}
\author{Bachir El Khadir$^{*,1}$, %
		Jake Varley$^{2}$, %
		Vikas Sindhwani$^{2}$\\
		$^1$ORFE, Princeton University \quad $^2$Google Brain Robotics
		\thanks{$^*$Work done during internship at Google Brain Robotics, NYC}}
\begin{document}
\maketitle

%%%%%%%%%%%%%%%%%%%%%%%%%%%%%%%%%%%%%%%%%%%%%%%%%%%%%%%%%%%%%%%%%%%%%%%%%%%%%%%%
\begin{abstract}
Learning to effectively imitate human teleoperators, with generalization to unseen and dynamic environments, is a promising path to greater autonomy enabling robots to steadily acquire complex skills from supervision. We propose a new motion learning technique rooted in contraction theory and sum-of-squares programming for estimating a control law in the form of a polynomial vector field from a given set of  demonstrations.  Notably, this vector field is provably optimal for the problem of minimizing imitation loss while providing continuous-time  guarantees on the induced imitation behavior. Our method generalizes to new initial and goal poses of the robot and can adapt in real-time to dynamic obstacles during execution, with convergence to teleoperator behavior within a well-defined safety tube.  We present an application of our framework for pick-and-place tasks in the presence of moving obstacles on a 7-DOF KUKA IIWA arm. The method compares favorably to other learning-from-demonstration approaches on benchmark handwriting imitation tasks.

% We develop new technique in Robotics for learning from demonstration based on estimating a polynomial vector fields with contraction properties. 

% - We use results of algebraic geometry to find the optimal polynomial vector field of a given degree satisfying the desired contraction properties.

% A application of polynomial optimization to a problem inspired by imitation learning and learning from demonstration in Robotics.
\end{abstract}

%%%%%%%%%%%%%%%%%%%%%%%%%%%%%%%%%%%%%%%%%%%%%%%%%%%%%%%%%%%%%%%%%%%%%%%%%%%%%%%%
\section{Introduction}

Teleoperation is enabling robotic systems to become pervasive in settings where full autonomy is currently out of reach~\cite{goldberg1995desktop, zhang2018deep, dragan2012formalizing}. Compelling applications include minimally invasive surgery~\cite{talamini2002robotic,taylor2016medical}, space exploration~\cite{washington1999autonomous}, remote vehicle operations~\cite{fong2001vehicle} and disaster relief scenarios~\cite{marturi2016towards}. A human teleoperator can control a robot through tasks that have complex semantics and are currently difficult to explicitly program or to learn to solve efficiently without supervision.
 
A downside of teleoperation is that it requires continuous
 error-free~\cite{atkeson2016happened} operator attention even for highly repetitive tasks. This problem can be addressed through Learning-from-Demonstrations (LfD) or Imitation Learning techniques~\cite{schaal1999imitation,billard2008robot} where a control law needs to be inferred from a small number of demonstrations. Such a law can then bootstrap data-efficient reinforcement learning for challenging tasks \cite{zhang2018deep}.  The demonstrator attempts to ensure that the robot’s motions capture the relevant semantics of the task rather than requiring the robot to understand the semantics.  The learnt control law should take over from the teleoperator and enable the robot to repeatedly execute the desired task even in dynamically changing conditions. For example, the origin of a picking task and the goal of a placing task may dynamically shift to configurations unseen during training, and moving obstacles may be encountered during execution. The latter is particularly relevant in collaborative human-robot workspaces where safety guarantees are paramount. In such situations, when faced with an obstacle, the robot cannot follow the demonstration path anymore and needs to recompute a new motion trajectory in real-time to avoid collision and still attempt to accomplish the desired task.
 
 \begin{figure}[t]
  \centering
	\centering
    \begin{subfigure}[t]{0.48\textwidth}
		\centering
		\includegraphics[width=1\textwidth]{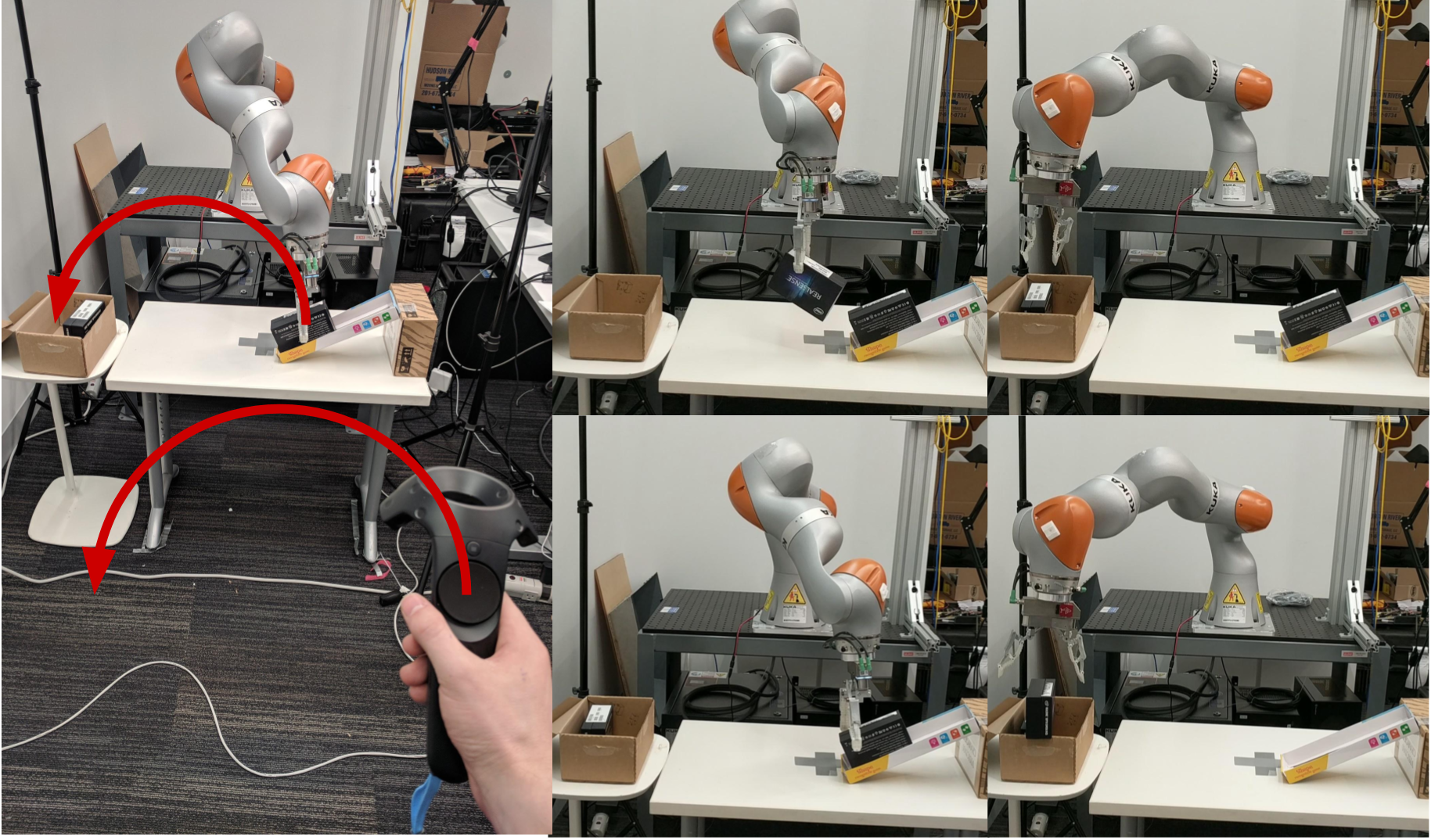}
		\caption{Pick and place teleoperation demonstration}
	\end{subfigure}
	\begin{subfigure}[t]{0.48\textwidth}
		\centering
		\includegraphics[width=1\textwidth]{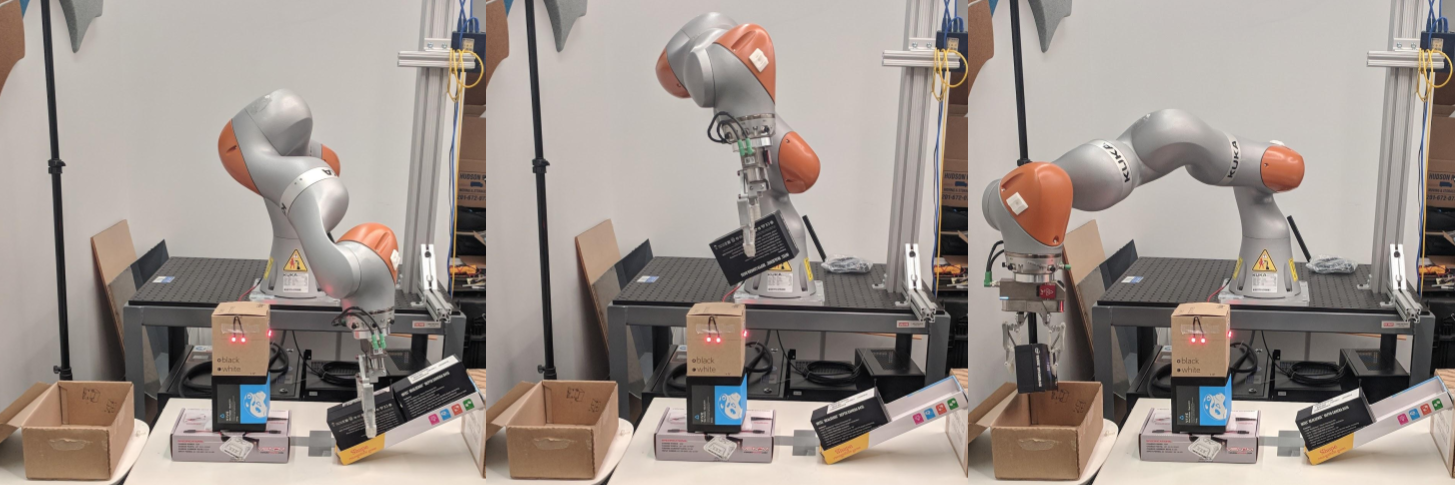}
		\caption{Pick and place via contracting vector fields with obstacle avoidance.}
		\label{pick_place_obstacle}
	\end{subfigure}
	\caption{(a) A non-technical user provides a demonstrates via teleoperation to accomplish a pick and place task. (b) The robot now autonomously executes the pick and place task with a contracting vector field (CVF) allowing for continuous time guarantees while also avoiding obstacles.}\vspace{-.5cm}
\end{figure}

Such real-time adaptivity can be elegantly achieved by associating demonstrations with a dynamical system~\cite{harish, sindhwani2018learning,CLFDM,khansari2011learning,khansari2017learning}: a vector field defines a closed-loop velocity control law. From any state that the robot finds itself in, the vector field can then steer the robot back towards the desired imitation behavior, without the need for path replanning with classical approaches. Furthermore, the learnt vector field can be modulated in real-time~\cite{khansari2012dynamical,slotine_obstacles,khatib1986real} in order to avoid collisions with obstacles.  

\begin{figure*}[ht]
    \centering
    \begin{subfigure}[t]{0.3\textwidth}
      \centering
      \includegraphics[trim={0 -0.45cm 0 0},clip, width=1\textwidth]{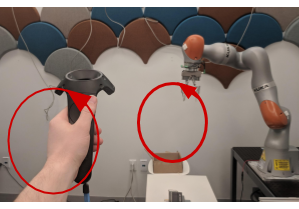}
      \caption{\label{fig:demo_traj} Demo trajectory $\samplex(t)$ started at (1,1) }
	\end{subfigure}\quad
    \begin{subfigure}[t]{0.3\textwidth}
      \centering
      \includegraphics[width=1\textwidth]{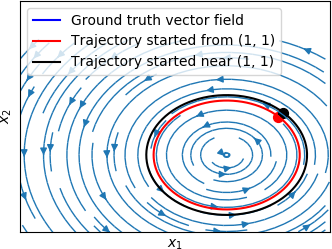}
      \caption{\label{fig:originalvf} True Vector Field $\originalvf$}
	\end{subfigure}\quad
	\begin{subfigure}[t]{0.3\textwidth}
	  \centering
      \includegraphics[width=1\textwidth]{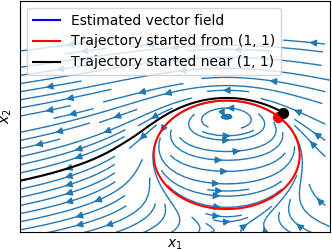}
      \caption{\label{fig:learnedvf} Estimated vector field $f^*$}
	\end{subfigure}
	\caption{ (a) a non-technical user demonstrates a circular trajectory. (b) the ``ground truth" vector field. (c) the estimated vector field. Both vector fields produce the same trajectory when started from $(1, 1)^T$ while they exhibit radically different behavior when started from a point arbitrarily close to $(1, 1)^T$.}
	\label{fig:illustration}\vspace{-.5cm}
\end{figure*}

At first glance, the problem of imitation learning of a \bachir{smooth} dynamical system, $\dot{\x} = f(\x)$ from samples $(\x, \dot{\x})$ appears to be a straightforward regression problem: simply minimize imitation loss $\sum_{i, t} \ltwonorm{f(\x^{(i)}(t)) - \dot \x^{(i)}(t)}^2$ over a suitable family of vector valued maps, $f\in {\cal F}$. However, a naive supervised regression approach may be woefully inadequate, as illustrated in the middle
panel of Figure~\ref{fig:illustration} where the goal is to have a KUKA arm imitate a circular trajectory. As can be seen, estimating vector fields from a small number of trajectories potentially leads to instability -- the estimated field easily diverges when the initial conditions are even slightly different from those encountered during training. Therefore, unsurprisingly, learning with stability constraints has been the technical core of existing dynamical systems based LfD approaches, e.g. see~\cite{khansari2011learning,khansari2017learning,harish,sindhwani2018learning}. However, these methods have one or more of the following limitations: (1) they involve non-convex optimization for dynamics fitting and constructing Lyapunov stability certificates respectively \bachir{and, hence,} have no end-to-end optimality guarantees, (2) the notion of stability is not trajectory-centric, but rather focused on reaching a single desired equilibrium point, and (3) they are computationally infeasible when formulated in continuous-time. With this context, our contributions in this paper include \bachir{the following:}
\begin{itemize}%[nosep]%[nosep,leftmargin=1em,labelwidth=*,align=left]
\item We formulate a novel continuous time optimization problem over vector fields involving an imitation loss subject to a generalization-enforcing constraint that turns the neighborhood of demonstrations into contracting regions~\cite{lohmiller1998contraction}. Within this region, all trajectories are guaranteed to coalesce towards the demonstration exponentially fast.
\item We show that our formulation leads to an instance of time-varying semidefinite programming~\cite{ahmadi2018tvsdp} for which a sum-of-squares relaxation~\cite{lasserre_global_2001,parrilo_semidefinite_2003,ahmadi2014towards} turns out to be exact! Hence, we can find the {\it globally optimal} polynomial  vector  field  that  has  the  lowest  imitation loss among  all  polynomial  vector  fields  of a given degree  that  are contracting on a region around the demonstrations in continuous time. 
\item On benchmark handwriting imitation tasks~\cite{lasa}, our method outperforms competing approaches in terms of a variety of imitation quality metrics.
\item We demonstrate our methods on a 7DOF KUKA pick-and-place LfD task where task completeness is accomplished despite dynamic obstacles in the environment, changing initial poses and moving destinations. By contrast, without contraction constraints, the  vector field tends to move far from the demonstrated trajectory activating emergency breaks on the arm and failing to complete the task.
\end{itemize}
%1) A novel contracting vector-field approach for LfD; 2) A comparison of various dynamical systems based movement generation techniques; 3) A method and analysis for embedding dynamic obstacles within the CVF for reactive obstacle avoidance; 4) A complete LfD system for controlling a real robotic arm via our CVF-approach with dynamic obstacle avoidance.
Our ``dirty laundry" includes: (1) we cannot handle high degree polynomials due to the scalability limitations of current SDP solvers, and (2) our notion of incremental stability is local, even though our method generalizes well in the sense that a wide contraction tube is setup around the demonstrations.

\section{Problem Statement}
\label{sec:problem-statement}

We are interested in estimating an unknown continuous time autonomous dynamical system 
\begin{equation}
\label{eq:ode}
\dot{\x} = \originalvf(\x),
\end{equation}
where $\originalvf: \mathbb R^n \rightarrow \mathbb R^n$ is an unknown continuously differentiable function.

We assume that we have access to one or several sample trajectories $\x^{(i)}:~[0, T]~\mapsto~\mathbb R^n$ that satisfy $\dot \x^{(i)} = \originalvf(\x^{(i)}) \; \forall t \in [0, T]$, where $T > 0$ and $i=1, \ldots, M$. These trajectories ($\x^{(i)}$, for $i=1 \ldots, M$) constitute our {\it training data}, and our goal is to search for an approximation of the vector field $\originalvf$ in a class of functions of interest $\mathcal F$ that reproduces trajectories as close as possible to the ones observed during training. In other words, we seek to solve the following continuous time least squares optimization problem (LSP):
\begin{equation}
  \label{eq:least_squares_problem}\tag{LSP}
  f^* \in \arg \min_{f \in \mathcal F}  \sum_{i=1}^M \int_{t=0}^T \ltwonorm{f(\x^{(i)}(t)) - \dot \x^{(i)}(t)}^2 \;\dt.
\end{equation}

%We note that problem~\eqref{eq:least_squares_problem} is a least squares problem, and therefore a convex quadratic optimization problem. As such, it could be solved for a large family of functions $\mathcal F$ (e.g. piece-wise linear functions, reproducing kernel hilbert spaces, neural networks \cite{zhang2018deep} ...) 

In addition to consistency with $\originalvf$, we want our learned vector field $f$ to generalize in conditions that were not seen in the training data. Indeed, the LSP problem generally admits multiple solutions, as it only dictates how the vector field should behave on the sample trajectories. This under-specification can easily lead to overfitting, especially if the class of function $\mathcal F$ is expressive enough. The example of Figure~\ref{fig:illustration} reinforces this phenomenon even for a simple circular motion. Note that standard data-independent regularization (e.g., $L_2$ regularizer) is insufficient to resolve the divergence illustrated here: a stronger stabilizer ensuring convergence, not just smoothness, of trajectories is needed. The notion of stability of interest to us in this paper is {\it contraction} which we now briefly review.

\iffalse
The following 2D example illustrates this issue. Consider the vector field given by: $$\originalvf(\x) = \begin{pmatrix}0&-1\\1&0\end{pmatrix}\x.$$ 
Let $\samplex(t)$ be the demonstration trajectory starting from $(1, 1)^T$ at $t=0$, i.e $\samplex(t) = \sqrt{2}(\cos(t+\frac{\pi}4), \sin(t+\frac{\pi}4))^T$.
The demonstration trajectory is shown in Figure \ref{fig:demo_traj} and overlaid by the vector field $\originalvf(\x)$ in Figure \ref{fig:originalvf}.

If we take $\mathcal F$ to be the space of all polynomial functions of degree at most $2$, then any vector field of the form 
$$f_{\iu}(\x) = \originalvf(\x) + (2 - \x_1^2 - \x_2^2)\iu$$
where $\iu$ is an arbitrary vector in $\R^2$ is a solution to the problem  \eqref{eq:least_squares_problem} as it matches $\originalvf$ completely on $\samplex$, i.e. $f_\iu(\samplex(t)) = \originalvf(\samplex(t))$ for all $t \in \R$.
The two vector fields on the other hand exhibit different qualitative behaviors on the rest of the space. For instance, some trajectories following the dynamics of $\originalvf$ starting from a point arbitrarily close to $(1, 1)$ are periodic, while those following $f^*$ diverge as shown in Figure \ref{fig:learnedvf}.

For this reason, it is desirable to impose additional constraints on the vector field $f$. One common constraint is {\it asymptotic stability}, which states that all trajectories close to some equilibrium point in $\mathbb R^n$ remain bounded and return to the equilibrium. [More motivation on contraction]
\fi
\subsection{Incremental Stability and Contraction Analysis}
\newcommand{\Metric}{{\bf M}}
\newcommand{\mnorm}[1]{{\|#1\|_{\Metric(\x)}}}

\label{sec:contraction} Notions of stability called {\it incremental stability} and associated contraction analysis tools~\cite{jouffroy2010tutorial,lohmiller1998contraction} are concerned with the convergence of system trajectories with respect to each other, as opposed to classical Lyapunov stability which is with respect to a single equilibrium. Contraction analysis derives sufficient and necessary conditions under which the displacement between any two trajectories will go to zero. We give in this section a brief presentation of this notion based on \cite{aylward2008stability}.

Contraction analysis of a system $\dot \x = f(\x)$ is best explained by considering the dynamics of $\delta \x(t)$, the infinitesimal displacement between two trajectories:
$$\delta \dot \x =  \jac f(\x) \delta \x \text{ where } \jac f(\x) = \frac{\partial}{\partial \x}f.$$
From this equation we can easily derive the rate of change of the infinitesimal squared distance between two trajectories $\ltwonorm{\delta \x}^2 = \delta \x^T\delta \x$ as follows:

\begin{equation}\label{eq:rate_infinitesimal_distance}
\frac{{\rm d}}{{\rm d}t} \|\delta \x\|_2^2 = 2 \delta \x^T\delta\dot \x = 2  \delta \x^T\jac f(\x) \delta \x.
\end{equation}

More generally, we can consider the infinitesimal squared distance with respect to a metric that is different from the Euclidian metric. A metric is given by smooth, matrix-valued function $\Metric:\R^+ \times \R^n \mapsto \R^{n \times n}$ that is uniformly positive definite, i.e. there exists $\varepsilon > 0$ such that
\begin{equation}
\label{eq:assumption_metric}
\Metric(t, \x) \succeq \varepsilon \id \quad \forall t \in \R^+, \; \forall \x \in \R^n,
\end{equation}
where $\id$ is the identity matrix and the relation $A \succeq B$ between two symmetric matrices $A$ and $B$ is used to denote that the smallest eigenvalue of their difference $A -B$ is nonnegative.
For the clarity of presentation, we only consider metric functions $\Metric(\x)$ that do not depend on time $t$.

The squared norm of an infinitesimal displacement between two trajectories with respect to this metric  is  given by $\mnorm{\delta\x}^2 \coloneqq \delta \x^T \Metric(\x) \delta \x$.
The Euclidean metric corresponds to the case where $\Metric(\x)$ is constant and equal to the identity matrix.
%For the clarity of the presentation in this paper, we focus our attention on metrics $\Metric(\x)$ that are uniformly  whose eigenvalues are uniformly bounded away from $0$ and $\infty$, i.e. we assume that there exits $\varepsilon > 0$ (independent of $\x$) such that
% \begin{equation}\label{eq:assumption_metric}
%  \|\iu\|_2^2 \le \mnorm{\iu}^2 \le C \|\iu\|_2^2 \quad \forall \x \in U \; \forall \iu \in \R^n.
% \end{equation}
%An equivalent way of stating the previous inequlaities is that the sphere of radius $1$ for the metric $\Metric(\x)$ contains the (Euclidian) sphere of radius $c$ and is contained in the (Euclidian) sphere of radius $C$.

Similarly to \eqref{eq:rate_infinitesimal_distance}, the rate of change of the squared norm of an infinitesimal displacement with respect to a metric $\Metric(\x)$ follows the following dynamics:
\begin{equation}\label{eq:rate_infinitesimal_m_distance}
\frac{{\rm d}}{{\rm d}t}\mnorm{\delta \x}^2 = \delta \x^T(\sym[\Metric(\x) \jac f(\x)] + \dot \Metric(\x))\delta \x,
\end{equation}
where $\sym[M]$ denotes $(M + M^T)/2$ for any square matrix $M$ and $\dot \Metric(\x)$ is the $n \times n$ matrix whose $(i,j)\text{-entry}$ is $\nabla \Metric_{ij}(\x)^T f(\x)$.
This motivates the following definition of {\it contraction}.
\begin{definition}[Contraction]
 For a positive constant $\tau$ and a subset $U$ of $\R^n$ the system $\dot \x = f(\x)$ is said to be $\tau\text{-\it{contracting}}$ on the region $U$ with respect to a metric $\Metric(\x)$ if 
\begin{equation}\label{eq:contraction_lmi}
\sym[\Metric(\x)\jac f(\x)] + \dot \Metric(\x) \preceq -\tau \Metric(\x) \quad \forall \x \in U.
\end{equation}
\end{definition}

\begin{remark}
When the vector field $f$ is a linear function $\dot{\x} = A \x$, and the metric $\Metric(\x)$ is constant $\Metric(\x)=P$, it is easy to see that contraction condition \eqref{eq:contraction_lmi} is in fact equivalent to global stability condition, \begin{equation}\label{eq:gas_linear_system_lmi}
    P \succ 0 \text{ and } \bachir{ \sym(PA^T) \preceq -\tau P.}
\end{equation}
\end{remark}

Given a $\tau\text{-\it{contracting}}$ vector field with respect to a metric $\Metric(\x)$, we can conclude from the dynamics in \eqref{eq:rate_infinitesimal_m_distance} that
$$\frac{{\rm d}}{\dt} \mnorm{\delta\x}^2 \le -\tau \mnorm{\delta\x}$$
Integrating both sides yields,
 $$\mnorm{\delta\x} \le e^{-\frac \tau 2 t} \mnorm{\delta\x(0)}$$

 Hence, any infinitesimal length $\mnorm{\delta\x}$ (and by assumption \eqref{eq:assumption_metric}, $\ltwonorm{\delta\x}$) converges exponentially to zero as time goes to infinity. This implies that in a contraction region, trajectories will tend to \bachir{converge together}  towards a nominal path. If the entire state-space is contracting and a finite equilibrium exists,
then this equilibrium is unique and all trajectories converge
to this equilibrium.

%  We summarize this observation in the following theorem.

% \begin{theorem}\bachir{make more formal}
% \label{thm:contraction_implies_convergence_trajectories}
% If the system $\dot \x = f(\x)$ is $\tau\text{-\it{contracting}}$ on a region $U \in \R^n$ with respect to a uniformly positive definite metric $\Metric$, then
% \begin{equation}
%     \mnorm{\x(t; \x_1) - \x(t; \x_2)} \le e^{-\frac{\tau}2t}\mnorm{\x_1-\x_2}
% \end{equation}
% for any two trajectories $\x(t; \x_1)$ and $\x(t; \x_2)$ that do not leave the region $U$.
% \end{theorem}

% Since $\ltwonorm{\cdot} \le \frac1{\varepsilon} \mnorm{\cdot}$, this theorem implies in particular that
% $$\ltwonorm{\x(t; \x_1) - \x(t; \x_2)} \rightarrow 0$$

In the next section, we explain how to {\it globally} solve the following continuous-time vector field optimization problem to fit a contracting vector field to the training data given some fixed metric $\Metric(\x)$.  We refer to this as the least squares problem with contraction (LSPC):

\begin{align}
\min_{f \in \mathcal F}& \sum_{i=1}^M\int_{t=0}^T \|f(\x^{(i)}(t)) - \dot{\x^{(i)}}(t)\|^2_2 \; \dt \label{eq:least_squares_with_contraction}\tag{LSPC}\\
  &\textrm{s.t. $f$ is contracting on a region $U \subseteq \R^n$}\nonumber\\
  &\textrm{containing the demonstrations $\x^{(i)}(t)$}\nonumber\\
  &\textrm{with respect to the metric $\Metric(\x)$.}\nonumber
\end{align}

The search for a contraction metric itself may be interpreted as the search for a Lyapunov function of the specific form $V(\x) = f(\x)^T \mathbf{M}(\x) f(\x)$.  As is the case with Lyapunov analysis in general, finding such an incremental stability certificate for a given dynamical system is a nontrivial problem; see~\cite{aylward2008stability} and references therein. If one wishes to find the vector field and a corresponding contraction metric at the same time, then the problem becomes non-convex. A common approach to handle this kind of problems is to optimize over one parameter at a time and fix the other one to its latest value and then alternate (i.e. fix a contraction metric and fit the vector field, then fix the vector field and improve on the contraction metric.)

\section{Learning Contracting Vector Fields as a Time-Varying Convex Problem}

In this section we explain how to formulate and solve the problem of learning a contracting vector field from demonstrations described in \eqref{eq:least_squares_with_contraction}. We will first see that we can formulate it as a {\it time-varying semidefinite problem}. We  will then describe how to use tools from {\it sum of squares programming} to solve it.

\subsection{Time-Varying Semidefinite Problems}
\newcommand{\lossTVC}{L}
\newcommand{\mapTVC}{\mathcal L}
We call time-varying semidefinite problems \eqref{eq:tv-convex-problem} optimization programs of the form
\begin{align}
  \min_{f \in \mathcal F}& \quad \lossTVC(f) \label{eq:tv-convex-problem}\tag{\text{TV-SDP}}\\
  \textrm{s.t.}& \quad \mapTVC_i f(t) \succeq 0 \quad \forall i=1,\ldots,m \; \forall t \in [0, T],\nonumber
\end{align}
where the variable $t \in [0, T]$ stands for time, the loss function  $\lossTVC: \mathcal F \mapsto \R$ in the objective is assumed to be convex and the $\mapTVC_i$ ($i=1,\ldots,m$) are linear functionals that map an element $f \in \mathcal F$ to a matrix-valued function $\mapTVC_i f: [0, T] \mapsto \R^{n\times n}$. 
%
%Example: \bachir{find (and solve?) simple example to tv-convex} The problem of finding the shortest path contained between two curves $p(t), q(t)$. $\mathcal F$ set of functions from $[0, T]$ to $\R$  .
%$C(t) = [p(t), q(t)]$. $\mapTVC$ the identity. $\lossTVC(f) = \int_0^T \sqrt{1 + f'(t)^2} \; {\rm d}t$. 
%
We will restrict the space of functions $\mathcal F$ to be the space of functions whose components are polynomials of degree $d \in \mathbb N$:
\begin{equation}
\label{eq:feasible_set_tvc}
\mathcal F \coloneqq \{ f: \R^n \mapsto \R^n \; | \; f_i \in \R_d[\x]\},
\end{equation}
and we make the assumption that $\mapTVC_i f$ is a matrix with polynomial entries.
Our interest in this setting stems from the fact that polynomial functions can approximate most functions reasonably well. Moreover, polynomials are suitable for algorithmic operations as we will see in the next section. See~\cite{ahmadi2018tvsdp} for a more in-depth treatment of time-varying semidefinite programs with polynomial data.

\newcommand{\trainingdata}{T}
\newcommand{\polytrainingdata}{{T^{\text{poly}}}}
\newcommand{\xni}{\x^{(i)}}
\newcommand{\dotxni}{\dot \x^{(i)}}
\newcommand{\xpoly}{\xni_{\text{poly}}}
\newcommand{\dotxpoly}{\dotxni_{\text{poly}}}

Let us now show how to reformulate the problem in \eqref{eq:least_squares_with_contraction} of fitting a vector field $f: \R^n \mapsto \R^n$ to $m$ sample trajectories $\{(\xni(t), \dotxni(t)) \; | \; t \in [0, T], \; i=1, \ldots, m\}$ as a~\eqref{eq:tv-convex-problem}. For this problem to fit within our framework, we start by approximating each trajectory $\xni(t)$ with a polynomial function of time $\xpoly(t)$. Our decision variable is the polynomial vector field $f$ and we seek to optimize the following objective function
\begin{equation}\label{eq:loss_cvf}
\lossTVC(f) \coloneqq \sum_{i=1}^M\int_{t=0}^T \|f(\xpoly(t)) - \dotxpoly(t)\|^2_2 \; {\rm d}t    
\end{equation} 
which is already convex (in fact convex quadratic). In order to impose the contraction of the vector field $f$ over some region around the trajectories in demonstration, we use a smoothness argument to claim that it is sufficient to impose contraction {\it only on} the trajectories themselves. See Proposition~\ref{lem:size_contraction_tube} later for a more quantitative statement of this claim. To be concrete, we take 
\begin{align}\label{eq:constraints_cvf}
\mapTVC_i f(\cdot) \coloneqq &-\sym[\Metric(\xpoly(\cdot))\jac f(\xpoly(\cdot))]\nonumber
\\&- \dot \Metric(\xpoly(\cdot)) - \tau \Metric(\xpoly(\cdot)),
\end{align}
where $\Metric(\x)$ is some known contraction metric.

%\bachir{Explain where this formulation comes from.}
%\bachir{continuous time guarantees}

%\bachir{Explain why polynomial data / solution}
% \begin{theorem}[Universality of polynomials]
% If there exists a vector field $f$ that $\tau$ contracting, then for any $\tau'$ there exists a degree $d$ such that Problem (1) will return a polyonomial vector field that is $\tau'$ contracting.
% \end{theorem}
\subsection{Sum-Of-Squares Programming}
In this section we review the notions of sum-of-squares (SOS) programming and its applications to polynomial optimization, and how we apply it for learning a contracting polynomial vector field. SOS techniques have found several applications in Robotics: constructing Lyapunov functions~\cite{ahmadi2012algebraic}, locomotion planning~\bachir{\cite{posa2017balancing}}, design and verification of provably safe controllers~\cite{majumdar2013control}, grasping and manipulation~\cite{dai2018synthesis}, inverse optimal control~\cite{pauwels2014inverse} and modeling 3D geometry~\cite{ahmadi2016geometry}.

Let $\R_d[\x]$ be the ring of polynomials $p(\x)$ in real variables $\x = (x_1, \ldots, x_n)$ with real coefficients of degree at most $d$. 
A polynomial $p \in \R[\x]$ is nonnegative if $p(\x) \ge 0$ for every $\x \in \R^n$. In many applications, including the one we cover in this paper, we seek to find the coefficients of one (or several) polynomials without violating some nonnegativity constraints. While the notion of nonnegativity is conceptually easy to understand, even testing whether a given polynomial is nonnegative is known to be NP-hard as soon as the degree $d \ge 4$ and the number of variables $n \ge 3$.

A polynomial $p \in \R_d[\x]$, with $d$ even, is a sum-of-squares (SOS) if there exists polynomials $q_1, \ldots, q_m \in \R_{\frac d2}[\x]$ such that
\begin{equation}\label{eq:sos}
    p(\x) = \sum_{i=1}^m q_i(\x)^2.
\end{equation}
An attractive feature of the set of SOS polynomials is that optimizing over it can be cast as a semidefinite program of tractable size, for which many solvers already exist. Indeed, it is known \cite{lasserre_global_2001}\cite{parrilo_semidefinite_2003} that a polynomial $p(\x)$ of degree $d$ can be decomposed as in \eqref{eq:sos} if and only if there exists a positive semidefinite matrix $Q$ such that
$$p(\x) = z(\x)^T Q z(\x) \; \forall \x \in \R^n,$$
where $z(\x)$ is the vector of monomials of $\x$ up to degree $\frac{d}2$, and the equality between the two sides of the equation is equivalent to a set of linear equalities in the coefficients of the polynomial $p(\x)$ and the entries of the matrix $Q$.

{\it Sum-of Squares Matrices}: If a polynomial $p(\x)$ is SOS, then it is obviously nonnegative, and the matrix $Q$ acts as a certificate of this fact, making it easy to check that the polynomial at hand is nonnegative for every value of the vector $\x$. In order to use similar techniques to impose contraction of a vector field, we need a slight generalization of this concept to ensure that a {\it matrix-valued} polynomial $P(\x)$ (i.e. a matrix whose entries are polynomial functions) is positive semidefinite (PSD) for all values of $\x$. We can equivalently consider the scalar-valued polynomial  $p(\x, \iu) \coloneqq \iu^TP(\x)\iu$, where $\iu$ is a $n\times1$ vector of new indeterminates, as positive semidefiniteness of $P(\x)$ is equivalent to the nonnegativity of $p(\x, \iu)$. If $p(\x, \iu)$ is SOS, then we say that $P$ is {\it a sum-of-squares matrix} (SOSM) \cite{kojima2003sums, gatermann2004symmetry, scherer2006matrix}. Consequently, optimizing over SOSM matrices is a tractable problem.

{\it Exact Relaxation}: A natural question here is how much we lose by restricting ourselves to the set of SOSM matrices as opposed the set of PSD matrices. In general, these two sets are quite different \cite{choi1975positive}.
In our case however, all the matrices considered are univariate as they depend only on the variable of time $t$. {\it It turns out that, in this special case, these two notions are equivalent!} 

\begin{theorem}[See e.g. \cite{choi1980real}]
  A matrix-valued polynomial $P(t)$ is PSD everywhere (i.e. $P(t) \succeq 0 \; \forall t \in \R$) if and only if the associated polynomial $p(t, \iu) \coloneqq \iu^T P(t) \iu$ is SOS.
\end{theorem}

The next theorem generalizes this result to the case where we need to impose PSD-ness only on the interval $[0, T]$ (as opposed to $t \in \R$.)

\begin{theorem}[See Theorem 2.5 of \cite{dette_matrix_2002}]
  A matrix-valued polynomial $P(t)$ of degree $d$ is PSD on the interval $[0, T]$ (i.e. $P(t) \succeq 0 \; \forall t \in [0, T]$) if and only if can be written as
  \[\left\{
  \begin{array}{lll}
  P(t) &= t V(t) + (T-t) W(t) & \textrm{if $\deg(P)$ odd,}\\
  P(t) &= V(t)   + t(T-t) W(t) & \textrm{if $\deg(P)$ even.}
  \end{array}\right.\]
  where $V(t)$ and $W(t)$ are SOSM. In the first case, $V(t)$ and $W(t)$ have degree at most $\deg(P) - 1$, and in the second case $V(t)$ (resp. $W(t)$) has degree at most $\deg(P)$ (resp. $\deg(P)-2$).
  When that is the case, we say that $P(t)$ is SOSM on $[0, T]$.
\end{theorem}

\subsection{Main Result and CVF-P}
The main result of this section is summarized in the following theorem that states that the problem of fitting a contracting polynomial vector field to polynomial data can be cast as a semidefinite program.

\begin{theorem}\label{thm:learn_cvf_as_sdp}
The following semidefinite program
\begin{align}
\label{eq:sos_program}
\min_{f \in \mathcal F}& \quad \sum_{i=1}^M \int_{t=0}^T \|f(\x_p^{(i)}(t)) - \dot{\x_p}^{(i)}(t)\|^2_2 \; \dt \tag{LSPC-SOS} \\
\textrm{s.t.}& \quad \mapTVC_i f  \text{ is SOSM on $[0, T]$ for $i=1,\ldots,M$}.\nonumber
\end{align}
with $\mathcal F$, $\mapTVC_i$, and $\lossTVC$ defined as in \eqref{eq:feasible_set_tvc}, \eqref{eq:constraints_cvf} and \eqref{eq:loss_cvf} resp. finds the polynomial vector field that has the lowest fitting error $\lossTVC(f)$ among all polynomial vector fields of degree $d$ that are contracting on a region containing the demonstrations $\x_p^{(i)}$.
\end{theorem}

To reiterate, the above sum-of-squares relaxation leads to no loss of optimality: the SDP above returns the globally optimal solution to the problem stated in~\ref{eq:least_squares_with_contraction}. Our numerical implementation uses the Splitting Conic Solver (SCS)~\cite{ocpb:16} for solving large-scale convex cone problems. 

\bachir{
\begin{remark}
Note that the time complexity of solving the SDP defined in \eqref{eq:sos_program} is bounded above by a polynomial function of the number of trajectories, the dimension $n$ of the space where they live, and the degree $d$ of the candidate polynomial vector field. In practice however, only small to moderate values for $n$ and $d$ can be solved for as the exponents appearing in this polynomial are prohibitively large. Significant progress has been made in recent years in inventing more scalable alternatives to SDPs based on {\it linear} and {\it second order cone} programming that can be readily applied to our framework~ \cite{ahmadi2014dsos}.
\end{remark}}

% \bachir{
% \begin{remark}
% The complexity of solving the SDP defined above in \eqref{eq:sos_program}, using the  interior point method for instance, although polynomial in the the number of demonstrations, the dimension of the space where the demonstrations $\x^{(i)}(t)$ live, and the degree of the candidate polynomial vector field $f(\x)$, 
% . For this reason, our approach is practical mostly for small to moderate values of $n$ and $d$. Several approaches exist to adress this limitation. [SOS and DSOS]
% \end{remark}}
  
For the rest of this paper, our approach will be abbreviated as CVF-P, standing for Polynomial Contracting Vector Fields.

\subsection{Generalization Properties}

% Small perturbations to a vector fields do not affect trajectories
% \begin{lemma}
% Let $\x(t)$ and $\y(t)$ be solutions to systems 
% $$\dot \x = f(\x) \text{ and } \dot \y = f(\y) + h(\y) \; \forall t \in [0, T],$$
% and let $U \in R^n$ be a region such that
% $\x(t), \y(t) \in U$, $\sup_{\x \in U} \ltwonorm{h(\x)} \le \varepsilon$,
% and $f$ is $L\text{-Lipchitz}$ on $U$, then 
% $$\|\x(t) - \y(t)\| \le (\varepsilon +  \|\x(0) - \y(0)\|)t \exp(Lt) \; \forall t \in [0, T].$$
% \end{lemma}

% \begin{proof}
% The proof is an easy application of Grönwall's inequality that states that if the following integral inequality holds for some function $u:~\mathbb R~\rightarrow~\R^+$:
% $$u(t) \le \alpha \int_0^t u(s) \; {\rm d}s + t \beta \; \forall t \in [0, T],$$
% then $u(t) \le \beta t\exp(\alpha t)  \; \forall t \in [0, T]$

% Indeed, let $u(t) \coloneqq \|y(t) - x(t)\|$, then
% \begin{align*}
% u(t) 
% &= \|\x(0) - \y(0)\| + \int_0^t \dot \x(s) - \dot\y(s) \; {\rm d}s \|
% \\&= \|\x(0) - \y(0)\| + \|\int_0^t f(\x(s)) - f(\y(s)) - h(\y(s))\|  \; {\rm d}s
% \\&\le  u(0) + \int_0^t \|f(\x(s)) - f(\y(s))\|  + \|h(\y(s))\|  \; {\rm d}s
% \\&\le u(0) + \int_0^t L \|x(s) - y(s)\| \; {\rm d}s + t\varepsilon 
% \\&\le u(0) + \int_0^t L u(s) \; {\rm d}s + t\varepsilon
% \end{align*}

% $$\|\x(t) - \y(t)\| \le (\varepsilon +  \|\x(0) - \y(0)\|)t \exp(Lt) \; \forall t \in [0, T].$$
% \end{proof}

The contraction property of CVF-P generalizes to a wider region in the state space. The next proposition shows that any sufficiently smooth vector field that is feasible for the problem stated in \ref{eq:sos_program} is contracting on a ``tube'' around the demonstrated trajectories.

\begin{proposition}[\label{lem:size_contraction_tube}A lower bound on the contraction tube]
If  $f:\Omega \subseteq \R^n \mapsto \R^n$ is a twice continuously differentiable vector field that satisfies
$$ -\sym[\Metric(\x(t))\jac f(\x(t))] - \dot \Metric(\x(t)) \succeq \tau \Metric(\x)  \quad \forall t \in [0, T]$$
where $\Omega$ is a compact region of $\R^n$, $\tau$ is a positive constant, $\Metric(\x)$ is a positive definite metric, and $\x: [0, T] \mapsto \R^n$ is a path, 
then $f$ is $\tau/2\text{-contracting}$ with respect to the metric $\Metric(\x)$ on the region $U$ defined by
$$U \coloneqq \{ \x(t) + \delta \; | \; t \in [0, T],\; \|\delta\|_2 \le \varepsilon \} \cap \Omega,$$
where $\varepsilon$ is positive scalar depending only $\tau$ and on the smoothness parameters of $f(\x)$ and $\Metric(\x)$ and is defined explicitly in Eqn.~\ref{eq:size_contraction_tube}.
\end{proposition}

For the proof we will need the following simple fact about symmetric matrices.
\begin{lemma}\label{lem:lip_smallest_eigen}
For any $n \times n$ symmetric matrices $A$ and $B$
$$|\lambda_{\min}(A) - \lambda_{\min}(B)| \le n \max_{ij} |A_{ij} - B_{ij}|,$$
where $\lambda_{\min}(\cdot)$ denotes the smallest eigenvalue function.
\end{lemma}

\begin{proof}[Proof of Proposition \ref{lem:size_contraction_tube}]
Let $f$, $\Metric$, $\Omega$ and $\tau$ be as in the statement of Proposition \ref{lem:size_contraction_tube}.
\bachir{Define $c \coloneqq \min_{\x \in \Omega} \lambda_{\min}(\Metric(\x))$.
Notice that since the metric $\Metric(\x)$ is uniformly positive definite, then $c > 0$.} Let us now define
\begin{equation}\label{eq:size_contraction_tube}
\varepsilon \coloneqq \frac{\tau c}{2nK} > 0
\end{equation}
where $K$ is the scalar equal to
$$\max_{1 \le i,j\le n} \sup_{\x \in \Omega} \|\frac{\partial}{\partial \x} \left(\sym[\Metric(\x)\jac f(\x)] + \dot \Metric(\x) - \frac{\tau}{2}\Metric(\x)\right)_{ij}\|_2.$$

Fix $t \in [0, T]$, and let $\delta$ be a vector in $\R^n$ such that $\ltwonorm{\delta}~\le~\varepsilon$. Our aim is to prove that the matrix $R^\delta$ defined by
$$-\sym[\Metric(\x(t)+\delta)\jac f(\x(t)+\delta)] - \dot \Metric(\x(t)+\delta) -  \frac{\tau}{2} \Metric(\x(t)+\delta)$$
is positive semidefinite.
Notice that our choice for $K$ guarantees that the maps $\delta \mapsto R^\delta_{ij}$ are $L\text{-Lipchitz}$ for $i, j =1,\ldots,n$, therefore 
$\max_{ij} |R^\delta_{ij} - R^0_{ij}| \le K\varepsilon$. Using Lemma \ref{lem:lip_smallest_eigen} we conclude that the smallest eigenvalues of $R^{\delta}$  and $R^0$ are within a distance of \bachir{$nK\varepsilon$} of each other.
\bachir{Since we assumed that $R^0 \succeq \frac{\tau}2 \Metric(\x(t))$, then $\lambda_{\min}(R^0)$ is at least $c \frac{\tau}2$, and therefore $\lambda_{\min}(R^\delta)$ is at least $c\frac{\tau}{2} - nK\varepsilon$. We conclude that our choice of $\varepsilon$ in \eqref{eq:size_contraction_tube} guarantees that $R^\delta$ is positive semidefinite.
}
\end{proof}

We note that the estimate obtained in this proposition is quite conservative. In practice the contraction tube is much larger than what is predicted here.

%We want to prove that 
%$\tilde f(t, \x) = f(\x) + \vfobstacle(t, \x)$ is not that different from $f(\x)$. For that we consider the virtual system
%$$\dot \y = f(\y) + \vfobstacle(t, \x).$$
%This system is contracting, so if there is a solution of this system that converges to the nominal path, then so does $\x$.\bachir{}

\section{Empirical Comparisons: Handwriting Imitation}

We evaluate our methods on the LASA library of two-dimensional human handwriting motions commonly used for benchmarking dynamical systems based movement generation techniques in imitation learning settings~\cite{khansari2017learning}\cite{lemme2015open}\cite{harish}. This dataset contains $30$ handwriting motions recorded with a pen input on a Tablet PC.  For each motion, the user was asked to draw 7  demonstrations of a desired pattern, by starting from different initial positions and ending at the same final point. Each demonstration trajectory comprises of $1000$ position ($\x$) and velocity ($\dot{\x}$) measurements. We use 4 demonstrations for training and 3 demonstrations for testing as shown in Figure~\ref{fig:lasa}.

\begin{figure}[h!]
\centering
\includegraphics[height=3cm, width=0.45\linewidth]{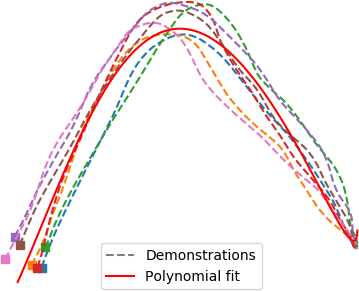}~
\includegraphics[height=3cm, width=0.45\linewidth]{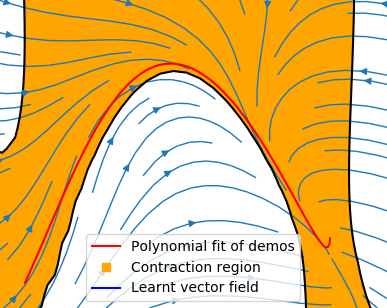}
\caption{\label{fig:lasa}\bachir{ The figure on the left shows demonstration trajectories (dotted) and the polynomial fit of the demonstrations (solid line) for the {\it Angle} shape. The figure on the right visualizes both the polynomial fit (red), the learnt vector field (blue), and the contraction region (orange) for the incrementally stable vector field learned using our method.}}
\end{figure}

\begin{table}[h]\footnotesize
  \bachir{
    \begin{tabular}{|c|c|c|c|c|} 
      \hline
      Metric & DMP & SEDS & CLFDM & \textbf{CVF-P}\\
      \hline
      \multicolumn{5}{ |c| }{\bf Reproduction Accuracy}\\
      \hline
      TrainingTrajectoryError   & 4.1 & 7.2  & 4.9 & 6.5\\
      \hline                                                                
      TrainingVelocityError     & 7.4 & 14.6 & 11.0 & 13.9\\
      \hline                                                                
      TestTrajectoryError       & 5.5 & 4.6  & 12.2 & 3.8 \\
      \hline                                                                
      TestVelocityError         & 8.7 & 11.3 & 15.5 & 11.4\\
      \hline
      \multicolumn{5}{ |c| }{\bf Stability}\\
      \hline
      DistanceToGoal            & 3.6 & 3.2 & 6.7 & 2.5\\
      \hline                    
      DurationToGoal            & -   & 3.9 & 4.3 & 3.3\\
      \hline                                                                
      NumberReachedGoal         & 0/7 & 7/7 & 7/7 & 7/7\\
      \hline                                                                
      GridDuration (sec)          & 5.9 & 3.7 & 9.7 & 1.9\\
      \hline                                                                
      GridFractionReachedGoal   & 6\% & 100\% & 100\% & 100\%\\
      \hline                                                                
      GridDistanceToGoal        & 3.3 & 1.0 & 1.0 & 1.0 \\
      \hline               
      GridDTWD ($\times 10^{4}$)& 2.4 & 1.4 & 1.4 & 2.0 \\
      \hline
      \multicolumn{5}{ |c| }{{\bf Training and Integration Speed} (in seconds)}\\
      \hline
      TrainingTime         & 0.05 & 2.1 & 2.8 & 0.2 \\
      \hline
      IntegrationSpeed     & 0.21 & 0.06 & 0.15 & 0.01 \\
      \hline
    \end{tabular}
    }
    \vspace{.2cm}
  \caption{\label{tbl:lasa}LASA Angle shape benchmarks. Our approach is CVF-P.}
\end{table}

We report in Table \ref{tbl:lasa} comparisons on the {\it Angle} shape against state of the art methods for estimating stable dynamical systems, the Stable Estimator of Dynamical Systems (SEDS)~\cite{khansari2011learning}, Control Lyapunov Function-based Dynamic Movements (CLFDM)~\cite{CLFDM} and Dynamic Movement Primitives (DMP)~\cite{ijspeert2013dynamical}. The training process in these methods involves non-convex optimization with no global optimality guarantees. Additionally, DMPs can only be trained from one demonstration one degree-of-freedom at a time. For all experiments, we learn degree 5 CVF-Ps with $\tau=1.0$ and $\Metric(\x) = \id$. We report the following imitation quality metrics.
\par {\bf Reproduction Accuracy}: How well does the vector field reproduce positions and velocities in training and test demonstrations, when started from same initial conditions and integrated for the same amount of time as the human movement duration ($T$).   Specifically, we measure reproduction error with respect to $m$ demonstration trajectories as,
\begin{eqnarray*}
\textrm{TrajectoryError} = \frac{1}{m}\sum_{i=1}^m\frac{1}{T_i}\sum_{t=0}^{T_i} \|\x^i(t) - \hat{\x}^i(t)\|_2\\
\textrm{VelocityError} = \frac{1}{m}\sum_{i=1}^m\frac{1}{T_i}\sum_{t=0}^{T_i} \|\dot{\x}^i(t) - \hat{\dot{\x}}^i(t)\|_2.
\end{eqnarray*}
The metrics {\it TrainingTrajectoryError, TestTrajectoryError, TrainingVelocityError, TestVelocityError} report these measures with respect to training and test demonstrations. At the end of the integration duration ($T$), we also report {\it DistanceToGoal}: how far the final state is from the goal (origin).  Finally, to account for the situation where the learnt dynamics is somewhat slower than the human demonstration, we also generate trajectories for a much longer time horizon ($30T$) and report ${\it DurationToGoal}$: the time it took for the state to enter a ball of radius $1mm$ around the goal, and how often this happened for the $7$ demonstrations ({\it NumReachedGoal}). 
\par {\bf Stability}:  To measure stability properties, we evolve the dynamical system from $16$ random positions on a grid enclosing the demonstrations for a long integration time horizon (30T). We report the fraction of trajectories that reach the goal ($GridFraction$); the mean duration to reach the goal when that happens ($GridDuration$); the mean distance to the Goal ($GridDistanceToGoal$) and the closest proximity of the generated trajectories to a human demonstration, as measured using Dynamic Time Warping Distance ($GridDTWD$)~\cite{keogh2005exact} (since in this case trajectories are likely of lengths different from demonstrations).
\par {\bf Training and Integration Speed}: We measure both training time as well as time to evaluate the dynamical system which translates to integration speed.
%\end{compactitem}
 
\begin{figure}[t]
    \centering
    \includegraphics[height=5cm]{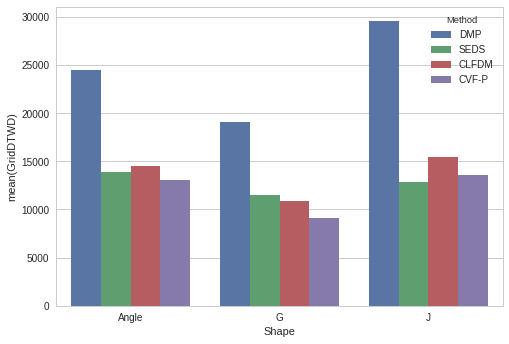}
    \caption{GridDTWD comparison on Angle, G and J shapes.}
    \label{table_grid_dtwd_metric}
\end{figure}

 It can be seen that our approach is highly competitive on most metrics: reproduction quality, stability, and training and inference speed. In particular, it returns the best mean dynamic time warping distance to the demonstrations when initialized from points on a grid. A comparison of {\it GridDTWD} on a few other shapes is shown in Figure. \ref{table_grid_dtwd_metric}.

\section{Pick-and-Place with Obstacles}

\begin{figure}[t]
  \centering
    \begin{subfigure}[t]{0.15\textwidth}
		\centering
		\includegraphics[width=1\textwidth]{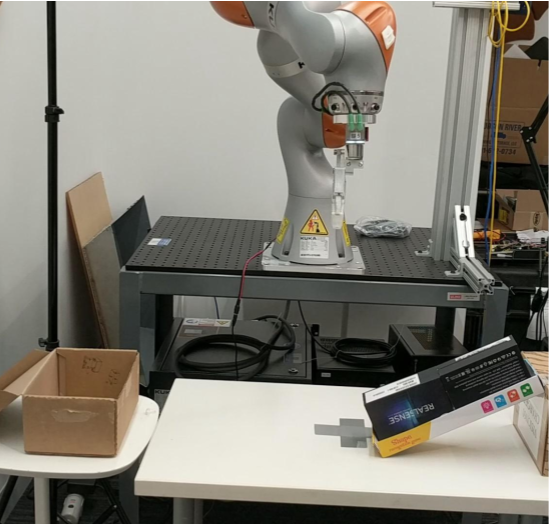}
		\caption{Home}
	\end{subfigure}
	\begin{subfigure}[t]{0.15\textwidth}
		\centering
		\includegraphics[width=1\textwidth]{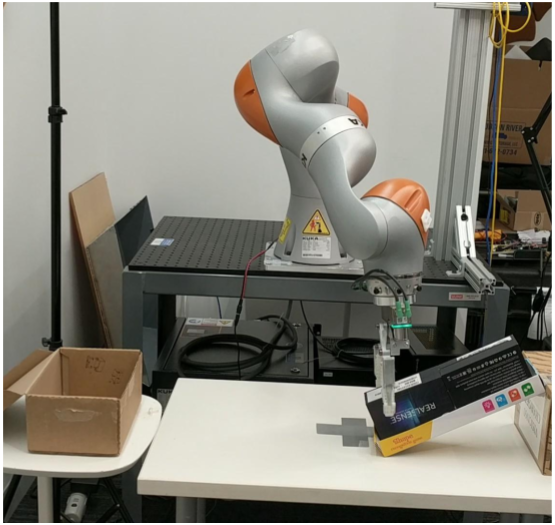}
	    \caption{Pick}
	\end{subfigure}
	\begin{subfigure}[t]{0.15\textwidth}
		\centering
		\includegraphics[width=1\textwidth]{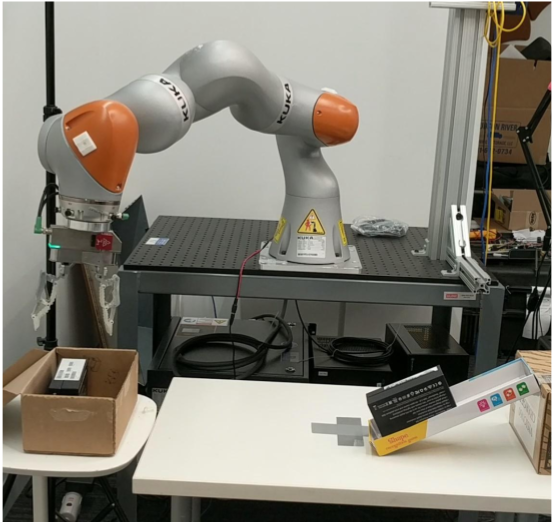}
		\caption{Place}
	\end{subfigure}
	\caption{In our task, the robot must move between the (a) home to (b) pick, (c) a place positions.}
	\label{pick_place_task_setup}\vspace{-.5cm}
\end{figure}

We consider a kitting task shown in Figure \ref{pick_place_task_setup} where objects are picked from a known location and placed into a box.  A teleoperator quickly demonstrates a few trajectories guiding a 7DOF KUKA IIWA arm to grasp objects and place them in a box.  After learning from demonstrations, the robot is expected to continually fill boxes to be verified and moved by a human working in close proximity freely moving obstacles in and out of the workspace. The arm is velocity-controlled in joint space at $50$ Hz.

\begin{figure*}[t]
  \centering
      \begin{subfigure}[t]{0.22\textwidth}
		\centering
		\includegraphics[trim={0 0 1cm 1cm},clip, width=.8\textwidth]{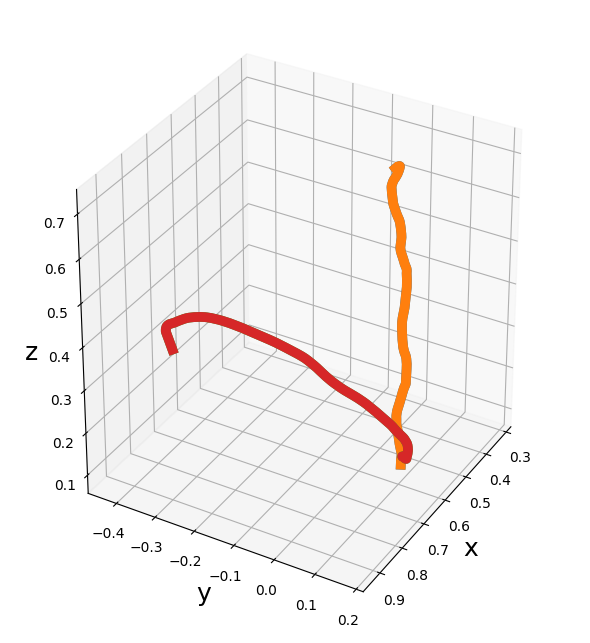}
		\caption{Demonstration Trajectory}
		\label{demo_trajectory}
	\end{subfigure}
    \begin{subfigure}[t]{0.22\textwidth}
		\centering
		\includegraphics[trim={0 0 1cm 1cm},clip, width=.8\textwidth]{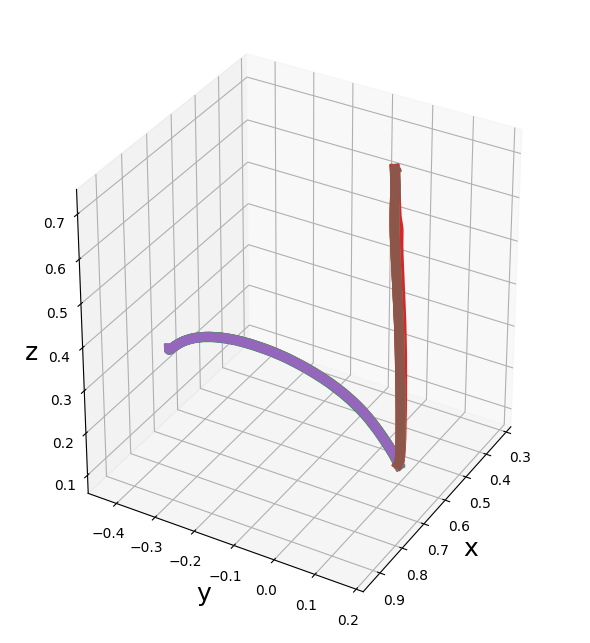}
		\caption{CVF-P}
		\label{pick_place_cvf-run}
	\end{subfigure}
	\begin{subfigure}[t]{0.22\textwidth}
		\centering
		\includegraphics[trim={0 0 1cm 1cm},clip, width=.8\textwidth]{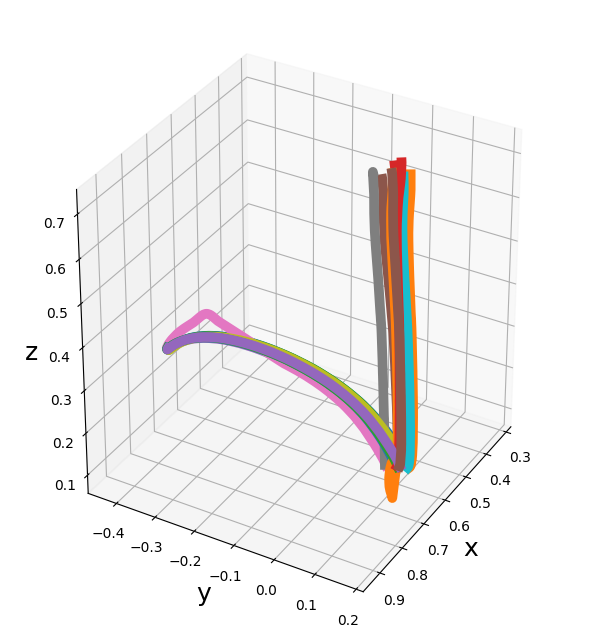}
		\caption{CVF-P, 0.05 Noise}
		\label{pick_place_cvf-run-small-noise}
	\end{subfigure}
    \begin{subfigure}[t]{0.22\textwidth}
		\centering
		\includegraphics[trim={0 -1.75cm 0 0},clip, width=.8\textwidth]{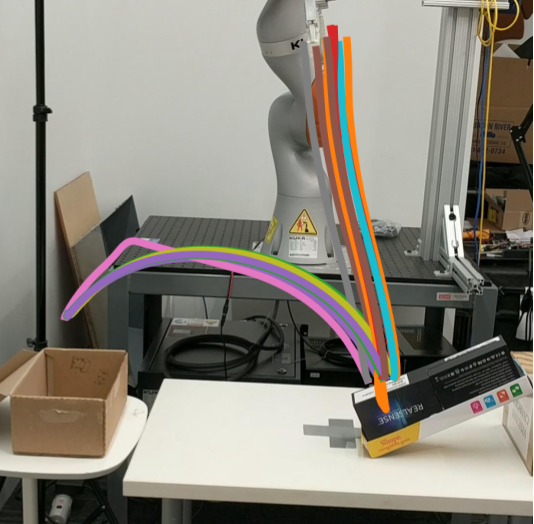}
		\caption{CVF-P, 0.05 Noise}
		\label{pick_place_cvf-run-small-noise-overlay}
	\end{subfigure}
	\begin{subfigure}[t]{0.22\textwidth}
		\centering
		\includegraphics[trim={0 0 1cm 1cm},clip, width=.8\textwidth]{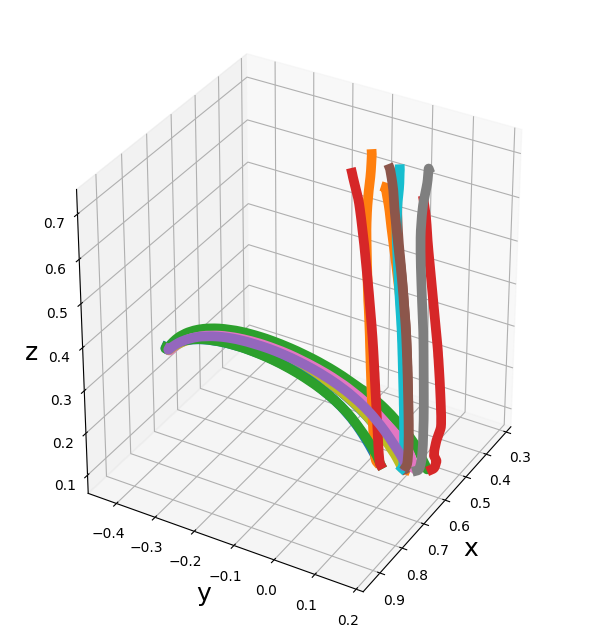}
		\caption{CVF-P, 0.1 Noise}
		\label{pick_place_cvf-run-big-noise}
	\end{subfigure}
	\begin{subfigure}[t]{0.22\textwidth}
		\centering
		\includegraphics[trim={0 0 1cm 1cm},clip, width=.8\textwidth]{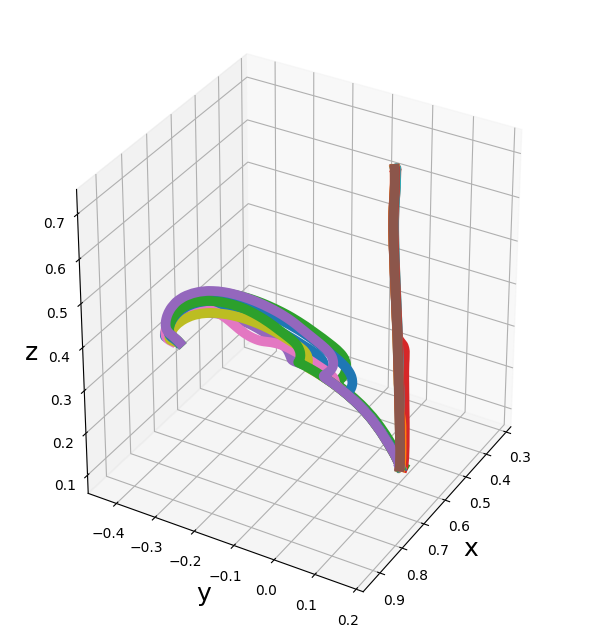}
	    \caption{CVF-P and Obstacle}
	    \label{pick_place_cvf-run-obstacles}
	\end{subfigure}
    \begin{subfigure}[t]{0.22\textwidth}
		\centering
		\includegraphics[trim={0 0 1cm 1cm},clip, width=.8\textwidth]{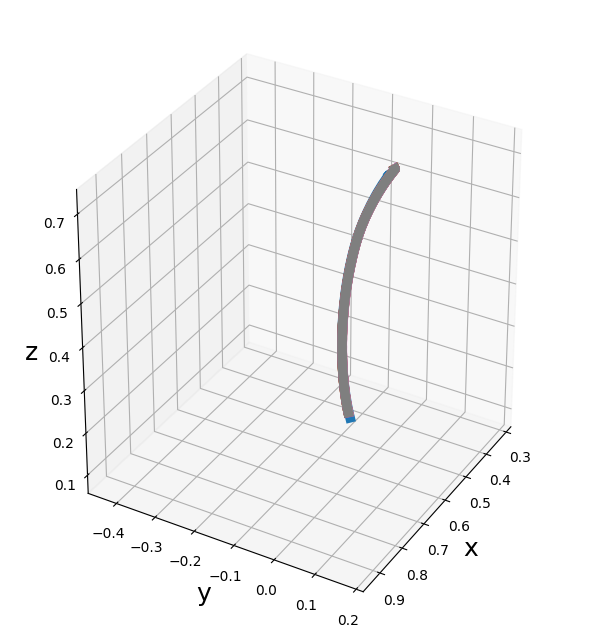}
		\caption{VF-P, No Contraction, \\No Noise, No Obstacle}
		\label{pick_place_cvf-run-no-noise-no-contraction}
	\end{subfigure}
	\begin{subfigure}[t]{0.22\textwidth}
		\centering
		\includegraphics[trim={0 0 1cm 1cm},clip, width=.8\textwidth]{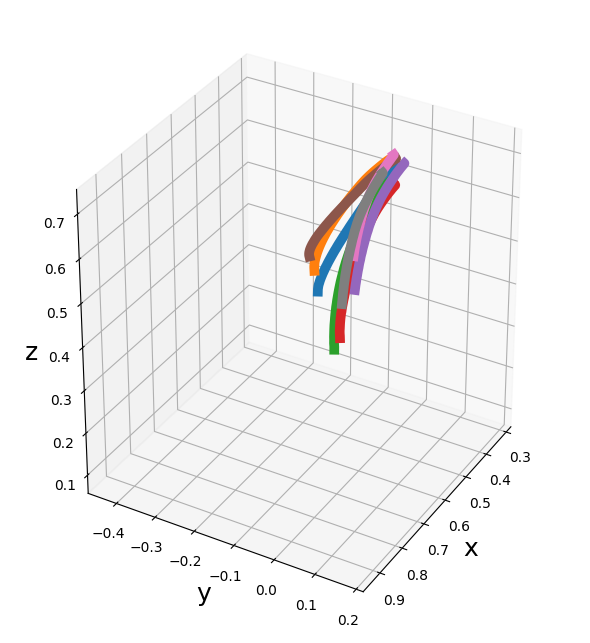}
	    \caption{VF-P, No Contraction, \\0.05 Noise, No Obstacle}
	    \label{pick_place_cvf-run-small-noise-no-contraction}
	\end{subfigure}
	\caption{(a) A user demonstrated trajectory visualization shows the path of the end effector through cartesian space. (b) Eight trajectories executed using a vectorfield in joint space learned from the demonstration. (c,d) Eight trajectories with uniform noise between [-0.05, 0.05] radians was added per-joint to the initial joint state. (e) Eight trajectories with uniform noise between [-0.1, 0.1] added to the initial joint state. (f) Eight new trajectories with an object in the way that modulates the learned vector field. Notice the motion deviates, and then returns to the desired trajectory. (g) Eight trajectories without contraction, the arm deviates from the demonstration and cannot complete the trajectory. (h) Eight trajectories without contraction and [-0.05, 0.05] noise, the arm cannot complete the trajectory.}
	\label{joint_space_trajectories}
\end{figure*}

\subsection{Demonstration Trajectory}
Figure \ref{demo_trajectory} shows the demonstration pick and place trajectory collected from the user. This trajectory was collected using an HTC Vive controller operated by a user standing in front of and watching the robot move through the demonstration as it is produced.  Different buttons on the remote were used to open/close the gripper, send the arm to the Home position, and indicate the start of a new trajectory.  The pick and place task was collected as two separate trajectories, one for the pick motion and another for the place motion. 

\subsection{Learning a Composition of Pick and Place CVF-Ps}
Using the demonstration trajectory, two different polynomial contracting vector fields (CVF-Ps) were fit to the data, one for the pick motion, one for the place.  These trajectories were fit to a degree 2 polynomial with $\tau=0.1$ and $\Metric(\x) = \id$, using an SCS solver run for 2500 iterations.  For the ease of visualization, we show the trajectories in cartesian space in Figure \ref{joint_space_trajectories}. The CVF-P was fit to the trajectory in the 7-dimensional joint space.  The arm was then run through using the vector field eight times starting from the home position. Each trajectory was allowed to run until the $L_2$-norm of the arm joint velocities dropped below a threshold of $0.01$. At that point, the arm would begin to move using the second vector field.  The trajectories taken by the arm are shown in Figure \ref{pick_place_cvf-run}.  The eight runs have very little deviation from each other.  

\subsection{Generalization to Different Initial Poses}
Next, noise is added to the home position of the arm, and again the vector field is used to move the arm through the task. Figure \ref{pick_place_cvf-run-small-noise} noise is added uniformly from the range [-0.05, 0.05] radians to each value of each joint of the arm's starting home position. Figure \ref{pick_place_cvf-run-small-noise-overlay}, shows these same trajectories overlaid on the Kuka arm. In Figure \ref{pick_place_cvf-run-big-noise} uniform noise is added in the same manner from the range $[-0.1, 0.1]$. Due to contraction,  trajectories are seen to converge from random initial conditions.   

\subsection{What happens without contraction constraints?}
In Figure \ref{pick_place_cvf-run-no-noise-no-contraction} the arm is run eight times using a vector field without contraction.  While the arm is consistent in the trajectory that it takes, the arm moves far from the demonstrated trajectory, and eventually causes the emergency break to activate at joint limits, failing to finish the task. 

In Figure \ref{pick_place_cvf-run-small-noise-no-contraction} The arm is again run eight times without contraction with noise added uniformly from the range [-0.05, 0.05] to each the value of each joint of the arm's starting home position.  The trajectory of the arm varies widely and had to be cut short as it was continually causing the emergency break to engage. 

\subsection{Whole-body Obstacle Avoidance}
\label{sec:dynamic_collision_avoidance}
Here we enable a Kuka robot arm to follow demonstrated trajectories while avoiding obstacles unseen during training. In the system we describe below,  collisions are avoided against any part of robot body. At every timestep, a commodity depth sensor like the Intel RealSense or PhaseSpace motion capture acquires a point cloud representation of the obstacle.  Our setup is along the lines of \cite{kappler2018real}, although we do not model occluded regions as occupied. At this point, our demonstrations and trajectories exist in joint space $\mathcal J \approx \mathbb R^7$, while our obstacle pointclouds exists in Cartesian space $\mathcal C \approx \mathbb R^3$ with an origin at the base of the robot.  
%We map the obstacle positions from $\mathcal C$ to $\mathcal J$ by inverting a precomputed map of occupied voxels in $\mathcal C$ to (potentially multiple) discrete positions in $\mathcal J$, and mark those positions as invalid.  These mapped positions are incorporated in a repulsive vector-field to push the arm away from collision as it moves,
\newcommand{\vfobstacle}{h^{\text{obstacles}}}
\newcommand{\ikmap}{\operatorname{IK}}
 
\subsubsection{Cartesian to Joint Space Map} We pre-compute a set-valued {\it inverse kinematic map} $\ikmap$ that maps a position $c \in \mathcal C$ to a subset of $\mathcal J$ containing all the joint configurations that would cause any part of the arm to occupy the position $c$. 
\begin{figure}[!h]
  \centering
  \includegraphics[width=0.975\linewidth]{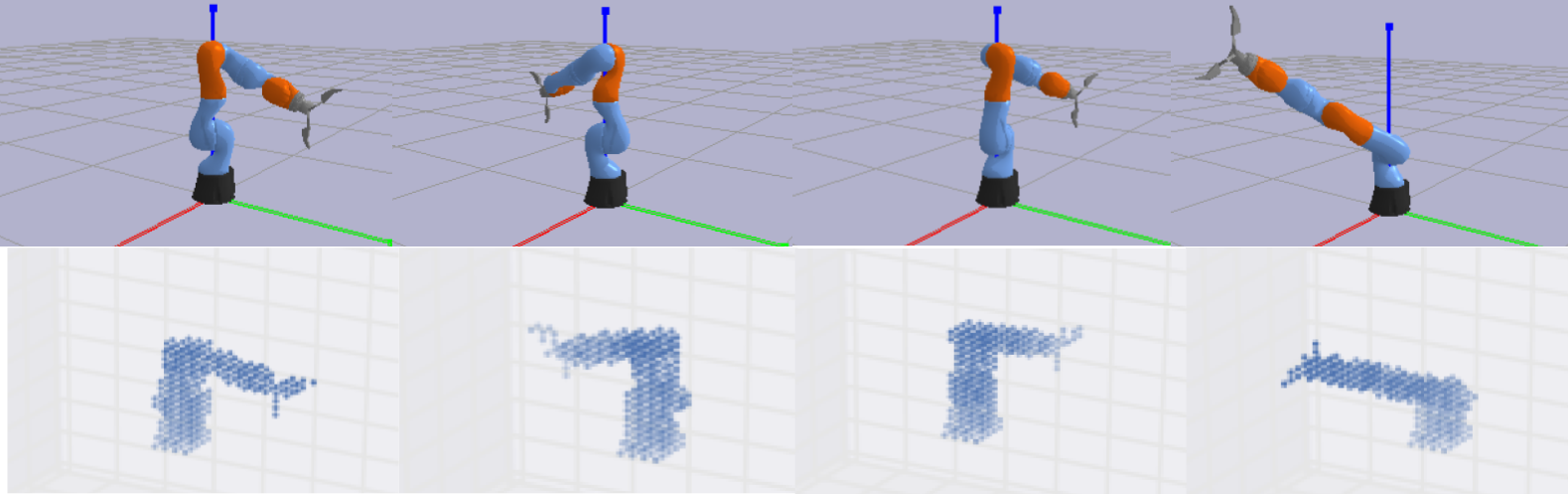}
  \caption{In order to produce a cartesian to joint space mapping, pybullet~\cite{coumans2017} was used to place the arm in over 658,945 configurations such as the 4 in the top row.  Then a voxelization of the arm was produced in this pose using binvox. }  
  \label{fig:voxelization}
\end{figure}

\newcommand{\simmap}{{\operatorname{FK}}}
More formally, the obstacles positions are known in Cartesian space $\mathcal C$ different from the control space $\mathcal J$ of the robot. (e.g. we control the joint angles rather than end-effector pose.)
The Kuka arm simulator allows us to query the forward kinematics map $\simmap: \mathcal J \rightarrow \mathcal C$. To compute the inverse of this map, the joint space of the robot was discretized into 658,945 discrete positions.  These discrete positions were created by regularly sampling each joint from a min to max angle using a step size of 0.1 radians. As shown in Figure \ref{fig:voxelization}, the robot was positioned at each point of the 658,945 discrete joint space points within pybullet\cite{coumans2017}, and the robot was voxelized using binvox\cite{min2004binvox}. This produced the map $\simmap$.  We then compute $\ikmap \coloneqq \simmap^{-1}$.

\subsubsection{Modulation of Contracting Vector Fields}
\begin{figure}[!h]
    \centering
    \includegraphics[height=3.0cm, width=\linewidth]{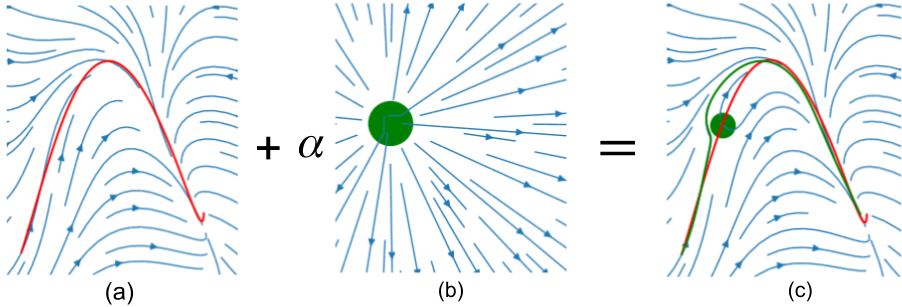}
    	\caption{\label{fig:explain_modulation_caption}(a) Shows a vector field $f$ learnt from a nominal path (red). (b) Depicts a repulsive vector field $\vfobstacle$ associated with an obstacle (green disk). (c) Shows modulated vector field $\tilde f$ (blue) plotted with a sample trajectory (green).}
    \label{fig:explain_modulation}\vspace{-.5cm}
\end{figure}
The obstacle positions are then incorporated in a repulsive vector-field to push the arm away from collision as it moves,
\begin{equation}\label{eq:modulation_function}
\vfobstacle(t, \x) \coloneqq \sum_{\substack{\text{positions of}\\\text{obstacles $c$} \\\text{at time $t$}}} \quad \sum_{j \in T^{-1}(c)} \frac{\x - j}{\ltwonorm{\x-j}^r},
\end{equation}
where the integer $r$ control how fast the effect of this vector field decays as a function of distance (a high value of $r$ makes the effect of $\vfobstacle$ local, while a small value makes its effect more uniform.)
\newcommand{\modulatedvf}{\tilde f}
This vector field is added to our learnt vector-field $f$ to obtain a {\it modulated vector field} (depicted in Figure \ref{fig:explain_modulation})
$$\modulatedvf(t, \x) = f(\x) + \alpha \; \vfobstacle(t, \x),$$
where $\alpha$ is positive constant that is responsible \bachir{for} controlling the strength of the modulation,
that is then fed to the Kuka arm. If the modulation is local and the obstacle is well within the joint-space contraction tube, we expect the motion to re-converge to the demonstrated behavior.

% Still debating whether this should be in the paper, maybe if we have some space left.
%\input{modulation_imgs.tex}

We point out that it is possible to use alternative modulation methods that come with different guarantees and drawbacks. In~\cite{khansari2012dynamical,slotine_obstacles} for instance, the authors use a multiplicative modulation function that preserves equilibrium points in the case of convex or concave obstacles.

While our approach does not enjoy the same guarantees, its additive nature allows us to handle a large number of obstacles as every term in Eqn.~\ref{eq:modulation_function} can be computed in a distributed fashion, and \bachir{furthermore}, we do not need to impose any restrictions on the shape of the obstacles (convex/concave). This is particularly important as our control space $\mathcal J$ is different from the space $\mathcal C$ where the obstacle are observed, and the map $\ikmap$ that links between the two spaces can significantly alter the shape of an obstacle in general (e.g. a sphere in cartesian space can be mapped to a disconnected set in joint space).

 %At the next timestep, the process repeats, the new point cloud associated with obstacles is quickly mapped to joint space, the new repulsive vector-field is computed and added to the LfD vector-field, and the arm continues to move via the combined vector field that pulls it along the demonstration trajectories while avoiding obstacles.
%\\\bachir{Explain the merits of this type of modulation: quick, easy to compute, scales well. Cite other types of modulation}\\

%\subsection{Obstacle avoidance via Modulation}
% \label{sec:Obstacle_avoidance_via_Modulation}
% A modulation is a function $\mathcal M: \mathcal F \rightarrow \mathcal F$ that satisfies the following properties
% - A path following $\mathcal Mf$ never enters the obstacle
% - $f - \mathcal Mf$ is close to 0 outside of a local ``modulation'' region.
% Examples of modulations include constructions in~\cite{khansari2012dynamical,slotine_obstacles}.
% %Examples
% %\begin{itemize}
% %\item Mohe's paper
% %\item Adding $\frac{-(x - c)}{\|x - c\|^{r+1}}$
% %\end{itemize}

%\subsection{Cartesian to Joint Space Map}
%\label{sec:cartesian_to_joint_space_map}

% We associate to every point in $T^{-1}(c)$ a repulsive vector field $f_c(x) = \frac{-(x - c)}{\|x - c\|^{r+1}}$ .
% Given a vector field $f$ and an obstacle at position $c \in \mathcal O$, we change $f$ to $\tilde f$ defined by

% $$\tilde f(x) = f(x) + \alpha \sum_{c \in T^{-1}(j)} f_c(x)$$

% \begin{definition}[Local Modulation]
% Modulation changes the vector field locally on  a compact set $S$.
% \end{definition}

%\section{Pick and Place Experiment}

\subsubsection{Real-time Obstacle Avoidance}
Here, using a \bachir{{\it real-time motion capture system}}, an obstacle is introduced to the robot's workspace as shown in Figure \ref{pick_place_obstacle}.  Eight trajectories were executed from the home position with the obstacle in the workspace, and the resultant trajectories are shown in Figure \ref{pick_place_cvf-run-obstacles}. At each timestep, the objects position was returned by the motion capture system. The point in Cartesian space was used to modulate the joint space vectorfield as described in Section \ref{sec:dynamic_collision_avoidance}. The tasks are accomplished as the arm avoids obstacles but remains within the joint-space contraction tube re-converging to the demonstrated behavior.

\section{Conclusion}
This work presents a novel approach to teleoperator imitation using contracting vector fields that are globally optimal with respect to loss minimization and providing continuous-time guarantees on the behaviour of the system when started from within a contraction tube around the demonstration.  Our approach compares favorably with other movement generation techniques. Additionally, we build a workspace cartesian to joint space map for the robot, and utilize it to update our CVF on the fly to avoid dynamic obstacles.  We demonstrate how this approach enables the transfer of knowledge from humans to robots for accomplishing a real world robotic pick and place task. Future work includes greater scalability of our solution, composition of CVFs for more complex tasks, integrating with a perception module and helping bootstrap data-hungry reinforcement learning approaches.

\bibliographystyle{plain}
\bibliography{citations}

\begin{thebibliography}{10}

\bibitem{lasa}
\url{https://cs.stanford.edu/people/khansari/download.html}.

\bibitem{ahmadi2012algebraic}
Amir~Ali Ahmadi.
\newblock Algebraic relaxations and hardness results in polynomial optimization
  and lyapunov analysis.
\newblock {\em arXiv preprint arXiv:1201.2892}, 2012.

\bibitem{ahmadi2018tvsdp}
Amir~Ali Ahmadi and Bachir El~Khadir.
\newblock Time-varying semidefinite programs.
\newblock {\em arXiv preprint arXiv:1808.03994}, 2018.

\bibitem{ahmadi2016geometry}
Amir~Ali Ahmadi, Georgina Hall, Ameesh Makadia, and Vikas Sindhwani.
\newblock Geometry of 3d environments and sum of squares polynomials.
\newblock {\em arXiv preprint arXiv:1611.07369}, 2016.

\bibitem{ahmadi2014dsos}
Amir~Ali Ahmadi and Anirudha Majumdar.
\newblock Dsos and sdsos optimization: Lp and socp-based alternatives to sum of
  squares optimization.
\newblock In {\em 2014 48th annual conference on information sciences and
  systems (CISS)}, pages 1--5. IEEE, 2014.

\bibitem{ahmadi2014towards}
Amir~Ali Ahmadi and Pablo~A Parrilo.
\newblock Towards scalable algorithms with formal guarantees for lyapunov
  analysis of control systems via algebraic optimization.
\newblock In {\em 2014 IEEE 53rd Annual Conference on Decision and Control
  (CDC)}, pages 2272--2281. IEEE, 2014.

\bibitem{atkeson2016happened}
Christopher~G Atkeson, BPW Babu, N~Banerjee, D~Berenson, CP~Bove, X~Cui,
  M~DeDonato, R~Du, S~Feng, P~Franklin, et~al.
\newblock What happened at the darpa robotics challenge, and why.

\bibitem{aylward2008stability}
Erin~M Aylward, Pablo~A Parrilo, and Jean-Jacques~E Slotine.
\newblock Stability and robustness analysis of nonlinear systems via
  contraction metrics and sos programming.
\newblock {\em Automatica}, 44(8):2163--2170, 2008.

\bibitem{billard2008robot}
Aude Billard, Sylvain Calinon, Ruediger Dillmann, and Stefan Schaal.
\newblock Robot programming by demonstration.
\newblock In {\em Springer handbook of robotics}, pages 1371--1394. Springer,
  2008.

\bibitem{choi1975positive}
Man-Duen Choi.
\newblock Positive semidefinite biquadratic forms.
\newblock {\em Linear Algebra and its Applications}, 12(2):95--100, 1975.

\bibitem{choi1980real}
Man-Duen Choi, Tsit-Yuen Lam, and Bruce Reznick.
\newblock Real zeros of positive semidefinite forms. {I}.
\newblock {\em Mathematische Zeitschrift}, 171(1):1--26, 1980.

\bibitem{coumans2017}
Erwin Coumans and Yunfei Bai.
\newblock pybullet, a python module for physics simulation, games, robotics and
  machine learning.
\newblock \url{http://pybullet.org/}, 2016--2017.

\bibitem{dai2018synthesis}
Hongkai Dai, Anirudha Majumdar, and Russ Tedrake.
\newblock Synthesis and optimization of force closure grasps via sequential
  semidefinite programming.
\newblock In {\em Robotics Research}, pages 285--305. Springer, 2018.

\bibitem{dette_matrix_2002}
Holger Dette and William~J. Studden.
\newblock Matrix measures, moment spaces and {Favard}'s theorem for the
  interval [0,1] and [0, $\infty$).
\newblock {\em Linear Algebra and its Applications}, 345(1-3):169--193, April
  2002.

\bibitem{dragan2012formalizing}
Anca~D Dragan and Siddhartha~S Srinivasa.
\newblock {\em Formalizing assistive teleoperation}.
\newblock MIT Press, 2012.

\bibitem{fong2001vehicle}
Terrence Fong and Charles Thorpe.
\newblock Vehicle teleoperation interfaces.
\newblock {\em Autonomous robots}, 11(1):9--18, 2001.

\bibitem{gatermann2004symmetry}
Karin Gatermann and Pablo~A Parrilo.
\newblock Symmetry groups, semidefinite programs, and sums of squares.
\newblock {\em Journal of Pure and Applied Algebra}, 192(1-3):95--128, 2004.

\bibitem{goldberg1995desktop}
Ken Goldberg, Michael Mascha, Steve Gentner, Nick Rothenberg, Carl Sutter, and
  Jeff Wiegley.
\newblock Desktop teleoperation via the world wide web.
\newblock In {\em Robotics and Automation, 1995. Proceedings., 1995 IEEE
  International Conference on}, volume~1, pages 654--659. IEEE, 1995.

\bibitem{slotine_obstacles}
Lukas Huber, Aude Billard, and Jean-Jacques E.~Slotine.
\newblock Avoidance of convex and concave obstacles with convergence ensured
  through contraction.
\newblock {\em IEEE Robotics and Automation Letters}, PP:1--1, 01 2019.

\bibitem{ijspeert2013dynamical}
Auke~Jan Ijspeert, Jun Nakanishi, Heiko Hoffmann, Peter Pastor, and Stefan
  Schaal.
\newblock Dynamical movement primitives: learning attractor models for motor
  behaviors.
\newblock {\em Neural computation}, 25(2):328--373, 2013.

\bibitem{jouffroy2010tutorial}
Jerome Jouffroy and Thor~I Fossen.
\newblock A tutorial on incremental stability analysis using contraction
  theory.
\newblock {\em Modeling, Identification and control}, 31(3):93, 2010.

\bibitem{kappler2018real}
Daniel Kappler, Franziska Meier, Jan Issac, Jim Mainprice, Cristina~Garcia
  Cifuentes, Manuel W{\"u}thrich, Vincent Berenz, Stefan Schaal, Nathan
  Ratliff, and Jeannette Bohg.
\newblock Real-time perception meets reactive motion generation.
\newblock {\em IEEE Robotics and Automation Letters}, 3(3):1864--1871, 2018.

\bibitem{keogh2005exact}
Eamonn Keogh and Chotirat~Ann Ratanamahatana.
\newblock Exact indexing of dynamic time warping.
\newblock {\em Knowledge and information systems}, 7(3):358--386, 2005.

\bibitem{khansari2011learning}
S~Mohammad Khansari-Zadeh and Aude Billard.
\newblock Learning stable nonlinear dynamical systems with gaussian mixture
  models.
\newblock {\em IEEE Transactions on Robotics}, 27(5):943--957, 2011.

\bibitem{CLFDM}
S.~Mohammad Khansari-Zadeh and Aude Billard.
\newblock Learning control lyapunov function to ensure stability of dynamical
  system-based robot reaching motions.
\newblock {\em Robotics and Autonomous Systems}, 6(62), 2014.

\bibitem{khansari2012dynamical}
Seyed~Mohammad Khansari-Zadeh and Aude Billard.
\newblock A dynamical system approach to realtime obstacle avoidance.
\newblock {\em Autonomous Robots}, 32(4):433--454, 2012.

\bibitem{khansari2017learning}
Seyed~Mohammad Khansari-Zadeh and Oussama Khatib.
\newblock Learning potential functions from human demonstrations with
  encapsulated dynamic and compliant behaviors.
\newblock {\em Autonomous Robots}, 41(1):45--69, 2017.

\bibitem{khatib1986real}
Oussama Khatib.
\newblock Real-time obstacle avoidance for manipulators and mobile robots.
\newblock {\em The international journal of robotics research}, 5(1):90--98,
  1986.

\bibitem{kojima2003sums}
Masakazu Kojima.
\newblock Sums of squares relaxations of polynomial semidefinite programs.
\newblock {\em Research report B-397, Dept. of Mathematical and Computing
  Sciences, Tokyo Institute of Technology}, 2003.

\bibitem{lasserre_global_2001}
J.~B. Lasserre.
\newblock Global optimization with polynomials and the problem of moments.
\newblock {\em SIAM Journal on Optimization}, 11(3):796--817, January 2001.

\bibitem{lemme2015open}
Andre Lemme, Yaron Meirovitch, Seyed~Mohammad Khansari-Zadeh, Tamar Flash, Aude
  Billard, and Jochen~J Steil.
\newblock Open-source benchmarking for learned reaching motion generation in
  robotics.
\newblock 2015.

\bibitem{lohmiller1998contraction}
Winfried Lohmiller and Jean-Jacques~E Slotine.
\newblock On contraction analysis for non-linear systems.
\newblock {\em Automatica}, 34(6):683--696, 1998.

\bibitem{majumdar2013control}
Anirudha Majumdar, Amir~Ali Ahmadi, and Russ Tedrake.
\newblock Control design along trajectories with sums of squares programming.
\newblock In {\em 2013 IEEE International Conference on Robotics and
  Automation}, pages 4054--4061. IEEE, 2013.

\bibitem{marturi2016towards}
Naresh Marturi, Alireza Rastegarpanah, Chie Takahashi, Maxime Adjigble, Rustam
  Stolkin, Sebastian Zurek, Marek Kopicki, Mohammed Talha, Jeffrey~A Kuo, and
  Yasemin Bekiroglu.
\newblock Towards advanced robotic manipulation for nuclear decommissioning: a
  pilot study on tele-operation and autonomy.
\newblock In {\em Robotics and Automation for Humanitarian Applications (RAHA),
  2016 International Conference on}, pages 1--8. IEEE, 2016.

\bibitem{min2004binvox}
Patrick Min.
\newblock Binvox, a 3{d} mesh voxelizer, 2004.

\bibitem{ocpb:16}
B.~O'Donoghue, E.~Chu, N.~Parikh, and S.~Boyd.
\newblock Conic optimization via operator splitting and homogeneous self-dual
  embedding.
\newblock {\em Journal of Optimization Theory and Applications},
  169(3):1042--1068, June 2016.

\bibitem{parrilo_semidefinite_2003}
Pablo~A. Parrilo.
\newblock Semidefinite programming relaxations for semialgebraic problems.
\newblock {\em Mathematical Programming}, 96(2):293--320, 2003.

\bibitem{pauwels2014inverse}
Edouard Pauwels, Didier Henrion, and Jean-Bernard~Bernard Lasserre.
\newblock Inverse optimal control with polynomial optimization.
\newblock {\em arXiv preprint arXiv:1403.5180}, 2014.

\bibitem{posa2017balancing}
Michael Posa, Twan Koolen, and Russ Tedrake.
\newblock Balancing and step recovery capturability via sums-of-squares
  optimization.
\newblock In {\em Robotics: Science and Systems}, pages 12--16, 2017.

\bibitem{harish}
Harish Ravichandar, Iman Salehi, and Ashwin Dani.
\newblock Learning partially contracting dynamical systems from demonstrations.
\newblock In {\em Conference on Robot Learning (CoRL)}, 2017.

\bibitem{schaal1999imitation}
Stefan Schaal.
\newblock Is imitation learning the route to humanoid robots?
\newblock {\em Trends in cognitive sciences}, 3(6):233--242, 1999.

\bibitem{scherer2006matrix}
Carsten~W Scherer and Camile~WJ Hol.
\newblock Matrix sum-of-squares relaxations for robust semi-definite programs.
\newblock {\em Mathematical programming}, 107(1-2):189--211, 2006.

\bibitem{sindhwani2018learning}
Vikas Sindhwani, Stephen Tu, and Mohi Khansari.
\newblock Learning contracting vector fields for stable imitation learning.
\newblock {\em arXiv preprint arXiv:1804.04878}, 2018.

\bibitem{talamini2002robotic}
Mark Talamini, Kurtis Campbell, and Cathy Stanfield.
\newblock Robotic gastrointestinal surgery: early experience and system
  description.
\newblock {\em Journal of laparoendoscopic \& advanced surgical techniques},
  12(4):225--232, 2002.

\bibitem{taylor2016medical}
Russell~H Taylor, Arianna Menciassi, Gabor Fichtinger, Paolo Fiorini, and Paolo
  Dario.
\newblock Medical robotics and computer-integrated surgery.
\newblock In {\em Springer handbook of robotics}, pages 1657--1684. Springer,
  2016.

\bibitem{washington1999autonomous}
Richard Washington, Keith Golden, John Bresina, David~E Smith, Corin Anderson,
  and Trey Smith.
\newblock Autonomous rovers for mars exploration.
\newblock In {\em Aerospace Conference, 1999. Proceedings. 1999 IEEE},
  volume~1, pages 237--251. IEEE, 1999.

\bibitem{zhang2018deep}
Tianhao Zhang, Zoe McCarthy, Owen Jowl, Dennis Lee, Xi~Chen, Ken Goldberg, and
  Pieter Abbeel.
\newblock Deep imitation learning for complex manipulation tasks from virtual
  reality teleoperation.
\newblock In {\em 2018 IEEE International Conference on Robotics and Automation
  (ICRA)}, pages 1--8. IEEE, 2018.

\end{thebibliography}

%\section*{Appendix}

\end{document}